\documentclass[twoside,11pt]{article}

\usepackage{jmlr2e}

\usepackage{amsmath}
\usepackage{amsfonts}
\usepackage{amssymb}
\usepackage{graphicx}
\usepackage{color}
\usepackage{units}

\usepackage{enumerate}
\usepackage{subcaption}

\usepackage{thmtools, thm-restate}

\usepackage{algorithm}
\usepackage[noend]{algorithmic}

\usepackage[capitalise]{cleveref}

\makeatletter
\def\renewtheorem#1{%
  \expandafter\let\csname#1\endcsname\relax
  \expandafter\let\csname c@#1\endcsname\relax
  \gdef\renewtheorem@envname{#1}
  \renewtheorem@secpar
}
\def\renewtheorem@secpar{\@ifnextchar[{\renewtheorem@numberedlike}{\renewtheorem@nonumberedlike}}
\def\renewtheorem@numberedlike[#1]#2{\newtheorem{\renewtheorem@envname}[#1]{#2}}
\def\renewtheorem@nonumberedlike#1{
\def\renewtheorem@caption{#1}
\edef\renewtheorem@nowithin{\noexpand\newtheorem{\renewtheorem@envname}{\renewtheorem@caption}}
\renewtheorem@thirdpar
}
\def\renewtheorem@thirdpar{\@ifnextchar[{\renewtheorem@within}{\renewtheorem@nowithin}}
\def\renewtheorem@within[#1]{\renewtheorem@nowithin[#1]}
\makeatother

\crefname{equation}{}{} % Default to eqref for equations
\crefname{section}{Sec.}{Sec.}
\crefname{enumi}{}{}
\newcommand{\T}{\mathrm{T}}                     % Transpose
\newcommand{\argmax}{\operatornamewithlimits{argmax}}

\newcommand{\bigO}{\mathcal{O}}
\newcommand{\algname}{Adaptive \textsc{GP-UCB}}
\newcommand{\algabbrv}{\textsc{A-GP-UCB}}

\newcommand*\dif{\mathop{}\!\mathrm{d}}
\newcommand{\mb}[1]{\mathbf{#1}}        % Bold font for variables
\newcommand{\added}[1]{#1}

\usepackage{lastpage}
\ShortHeadings{Unknown Hyperparameters}{Berkenkamp, Schoellig and Krause}
\firstpageno{1}

\begin{document}

% Fix cleveref for lemmas
\renewtheorem{lemma}{Lemma}

\title{No-Regret Bayesian Optimization with\\ Unknown Hyperparameters}

\author{\name Felix Berkenkamp \email befelix@inf.ethz.ch \\
       \addr Department of Computer Science\\
       ETH Zurich\\
       Zurich, Switzerland
       \AND
       \name Angela P. Schoellig \email schoellig@utias.utoronto.ca  \\
       \addr Institute for Aerospace Studies \\
       University of Toronto \\
       Toronto, Canada
       \AND
       \name Andreas Krause \email krausea@ethz.ch \\
       \addr Department of Computer Science\\
       ETH Zurich\\
       Zurich, Switzerland}

\editor{}

\maketitle

%!TEX root = ../root.tex

\begin{abstract}%   <- trailing '%' for backward compatibility of .sty file
Bayesian optimization (BO) based on Gaussian process models is a powerful paradigm to optimize black-box functions that are expensive to evaluate. While several BO algorithms provably converge to the global optimum of the unknown function, they assume that the hyperparameters of the kernel are known in advance. This is not the case in practice and misspecification often causes these algorithms to converge to poor local optima. In this paper, we present the first BO algorithm that is provably no-regret and converges to the optimum without knowledge of the hyperparameters. During optimization we slowly adapt the hyperparameters of stationary kernels and thereby expand the associated function class over time, so that the BO algorithm considers more complex function candidates. Based on the theoretical insights, we propose several practical algorithms that achieve the empirical sample efficiency of BO with online hyperparameter estimation, but retain theoretical convergence guarantees. We evaluate our method on several benchmark problems.
\end{abstract}


%!TEX root = ../root.tex

\section{Introduction}
\label{sec:introduction}

The performance of machine learning algorithms often critically depends on the choice of tuning inputs, e.g., learning rates or regularization constants. Picking these correctly is a key challenge. Traditionally, these inputs are optimized using grid or random search~\citep{Bergstra2012Random}. However, as data sets become larger the computation time required to train a single model increases, which renders these approaches less applicable. Bayesian optimization (BO,~\cite{Mockus2012Bayesian}) is an alternative method that provably determines good inputs within few evaluations of the underlying objective function. BO methods construct a statistical model of the underlying objective function and use it to evaluate inputs that are informative about the optimum. However, the theoretical guarantees, empirical performance, and data efficiency of BO algorithms critically depend on their own choice of hyperparameters and, in particular, on the prior distribution over the function space. Thus, we effectively shift the problem of tuning inputs one level up, to the tuning of hyperparameters of the BO algorithm.

In this paper, we use a Gaussian processes (GP, \citet{Rasmussen2006Gaussian}) for the statistical model. We present the first BO algorithm that does not require knowledge about the hyperparameters of the GP's stationary kernel and provably converges to the global optimum. To this end, we adapt the hyperparameters of the kernel and our BO algorithm, so that the associated function space grows over time. The resulting algorithm provably converges to the global optimum and retains theoretical convergence guarantees, even when combined with online estimation of hyperparameters.

\paragraph{Related work}

General BO has received a lot of attention in recent years. Typically, BO algorithms suggest inputs to evaluate by maximizing an acqusition function that measures informativeness about the optimum. Classical acquisition functions are the \textit{expected improvement} over the best known function value encountered so far given the GP distribution~\citep{Mockus1978Application} and the \textit{Upper Confidence Bound} algorithm,~\textsc{GP-UCB}, which applies the `optimism in the face of uncertainty' principle. The latter is shown to provably converge by~\cite{Srinivas2012Gaussian}. \citet{Durand2018Streaming} extend this framework to the case of unknown measurement noise. A related method is \textit{truncated variance reduction} by~\citet{Bogunovic2016Truncated}, which considers the reduction in uncertainty at candidate locations for the optimum. \cite{Hennig2012Entropy} propose \textit{entropy search}, which approximates the distribution of the optimum of the objective function and uses the reduction of the entropy in this distribution as an acquisition function. Alternative information-theoretic methods are proposed by~\citet{Hernandez-Lobato2014Predictive,Wang2017Maxvalue,Ru2018Fast}.
%
% Similar, but computationally less expensive, alternatives are \textit{predictive entropy search} by~\cite{Hernandez-Lobato2014Predictive}, \textit{max-value entropy search} by~\cite{Wang2017Maxvalue} and FITBO.
Other alternatives are the \textit{knowledge gradient}~\citep{Frazier2009Knowledgegradient}, which is one-step Bayes optimal, and \textit{information directed sampling} by~\cite{Russo2014Learning}, which considers a tradeoff between regret and information gained when evaluating an input. \cite{Kirschner2018Information} extend the latter framework to heteroscedastic noise.

These BO methods have also been successful empirically. In machine learning, they are used to optimize the performance of learning methods~\citep{Brochu2010Tutorial,Snoek2012Practical}. BO is also applicable more broadly; for example, in reinforcement learning to optimize a parametric policy for a robot~\citep{Calandra2014Experimental,Lizotte2007Automatic,Berkenkamp2016Bayesian} or in control to optimize the energy output of a power plant~\citep{Abdelrahman2016Bayesian}. It also forms the backbone of Google vizier, a service for tuning black-box functions~\citep{Golovin2017Google}.

Some of the previous BO algorithms provide theoretical guarantees about convergence to the optimum. These theoretical guarantees only hold when the kernel hyperparameters are known~\textit{a priori}. When this is not the case, hyperparameters are often inferred using either \textit{maximum a posteriori} estimates or sampling-based inference~\citep{Snoek2012Practical}. Unfortunately, methods that estimate the hyperparameters online are known to get stuck in local optima~\citep{Bull2011Convergence}. Instead, we propose to adapt the hyperparameters online in order to enlarge the function space over time, which allows us to provide guarantees in terms of convergence to the global optimum without knowing the hyperparameters. \citet{Wang2014Theoretical} analyze this setting when a lower bound on the kernel lengthscales is known \textit{a priori}. They decrease the lengthscales over time and bound the regret in terms of the known lower-bound on the lengthscales. Empirically, similar heuristics are used by~\citet{Wang2016Bayesian,Wabersich2016Advancing}. In contrast, this paper considers the case where the hyperparameters are \textit{not known}. Moreover, the scaling of the hyperparameters in the previous two papers did not depend on the dimensionality of the problem, which can cause the function space to increase too quickly.

Considering larger function classes as more data becomes available is the core idea behind structural risk minimization~\citep{Vapnik1992Principles} in statistical learning theory. However, there data is assumed to be sampled independently and identically distributed. This is not the case in BO, where new data is generated actively based on past information.

\paragraph{Our contribution}

In this paper, we present~\algname~ (\algabbrv), the first algorithm that provably converges to the globally optimal inputs when BO hyperparameters are \emph{unknown}.
Our method expands the function class encoded in the model over time, but does so slowly enough to ensure sublinear regret and convergence to the optimum.
Based on the theoretical insights, we propose practical variants of the algorithm with guaranteed convergence. Since our method can be used as an add-on module to existing algorithms with hyperparameter estimation, it achieves similar performance empirically, but avoids local optima when hyperparameters are misspecified. In summary, we:
\begin{itemize}
  \item Provide theoretical convergence guarantees for BO with unknown hyperparameters;
  \item Propose several practical algorithms based on the theoretical insights;
  \item Evaluate the performance in practice and show that our method retains the empirical performance of heuristic methods based on online hyperparameter estimation, but leads to significantly improved performance when the model is misspecified initially.
\end{itemize}

The remainder of the paper is structured as follows. We state the problem in~\cref{sec:problem_statement} and provide relevant background material in~\cref{sec:background}. We derive our main theoretical result in~\cref{sec:theory} and use insights gained from the theory to propose practical algorithms. We evaluate these algorithms experimentally in~\cref{sec:experiments} and draw conclusions in~\cref{sec:conclusion}. The technical details of the proofs are given in the appendix.

%!TEX root = ../root.tex

\section{Problem Statement}
\label{sec:problem_statement}

In general, BO considers global optimization problems of the form
\begin{equation}
    \mb{x}^* = \argmax_{\mb{x} \in \mathcal{D}} f(\mb{x}),
    \label{eq:optimize_f}
\end{equation}
where~$\mathcal{D} \subset \mathbb{R}^d$ is a compact domain over which we want to optimize inputs~$\mb{x}$, and~$f \colon \mathcal{D} \to \mathbb{R}$ is an objective function that evaluates the reward~$f(\mb{x})$ associated with a given input configuration~$\mb{x}$. For example, in a machine learning application,~$f(\mb{x})$ may be the validation loss and~$\mb{x}$ may be the tuning inputs (e.g., regularization parameters) of the training algorithm. We do not have any significant prior knowledge about the structure of~$f$. Specifically, we cannot assume convexity or that we have access to gradient information. Moreover, evaluations of~$f$ are corrupted by~$\sigma$-sub-Gaussian noise, a general class of noise models that includes, for example, bounded or Gaussian noise.

\paragraph{Regret}
We aim to construct a sequence of input evaluations~$\mb{x}_t$, that eventually maximizes the function value~$f(\mb{x}_t)$. One natural way to prove this convergence is to show that an algorithm has sublinear regret. The instantaneous regret at iteration~$t$ is defined as~$r_t = \max_{\mb{x} \in \mathcal{D}}f(\mb{x}) - f(\mb{x}_t) \geq 0$, which is the loss incurred by evaluating the function at~$\mb{x}_t$ instead of at the \textit{a priori unknown} optimal inputs. The cumulative regret is defined as~$R_T = \sum_{0 < t \leq T} r_t$, the sum of regrets incurred over~$T$ steps. If we can show that the cumulative regret is sublinear for a given algorithm, that is,~$\lim_{t \to \infty} R_t \,/ \, t = 0$, then eventually the algorithm evaluates the function at inputs that lead to close-to-optimal function values most of the time. We say that such an algorithm has~\emph{no-regret}. Intuitively, if the average regret approaches zero then, on average, the instantaneous regret must approach zero too, since~$r_t$ is strictly positive. This implies that there exists a~$t>0$ such that~$f(\mb{x}_t)$ is arbitrarily close to $f(\mb{x}^*)$ and the algorithm converges. Thus, we aim to design an optimization algorithm that has sublinear regret.

\paragraph{Regularity assumptions}

Without further assumptions, it is impossible to achieve sublinear regret on~\cref{eq:optimize_f}. In the worst case, $f$ could be discontinuous at every input in~$\mathcal{D}$. To make the optimization problem in~\cref{eq:optimize_f} tractable, we make regularity assumptions about~$f$. In particular, we assume that the function~$f$ has low complexity, as measured by the norm in a reproducing kernel Hilbert space (RKHS, \cite{Christmann2008Support}). An RKHS~$\mathcal{H}_k$ contains well-behaved functions of the form~$f(\mb{x}) = \sum_{i \geq 0} \alpha_i \, k(\mb{x}, \mb{x}_i)$, for given representer points~$\mb{x}_i \in \mathbb{R}^d$ and weights~$\alpha_i \in \mathbb{R}$ that decay sufficiently quickly. The kernel~$k(\cdot, \cdot)$ determines the roughness and size of the function space and the induced RKHS norm~$\|f\|_{k} = \sqrt{ \langle f, \, f \rangle }$ measures the complexity of a function~$f \in \mathcal{H}_k$ with respect to the kernel.

In the following, we assume that~$f$ in~\cref{eq:optimize_f} has bounded RKHS norm~$\|f\|_{k_\theta} \leq B$ with respect to a kernel~$k_\theta$ that is parameterized by hyperparameters~$\theta$. We write~$\mathcal{H}_\theta$ for the corresponding RKHS,~$\mathcal{H}_{k_\theta}$. For known~$B$ and~$\theta$, no-regret BO algorithms for~\cref{eq:optimize_f} are known, e.g., \textsc{GP-UCB}~\citep{Srinivas2012Gaussian}. In practice, these hyperparameters need to be tuned. In this paper, we consider the case where~$\theta$ and~$B$ are unknown.
We focus on stationary kernels, which measure similarity based on the distance of inputs, $k(\mb{x}, \mb{x}') = k(\mb{x} - \mb{x}')$. The most commonly used hyperparameters for these kernels are the lengthscales~$\theta \in \mathbb{R}^d$, which scale the inputs to the kernel in order to account for different magnitudes in the different components of~$\mb{x}$ and effects on the output value. That is, we scale the difference~$\mb{x} - \mb{x}'$ by the lengthscales~$\theta$,
\begin{equation}
  k_\theta(\mb{x}, \mb{x}') = k\left(
  \frac{[\mb{x}]_1 - [\mb{x}']_1 } { [\theta]_1 } ,
  \, \dots, \,
  \frac{[\mb{x}]_d - [\mb{x}']_d } { [\theta]_d } \right),
  \label{eq:stationary_lengthscale_kernel}
\end{equation}
where~$[\mb{x}]_i$ denotes the $i$th element of~$\mb{x}$. Typically, these kernels assign larger similarity scores to inputs when the scaled distance between these two inputs is small. \added{Another common hyperparameter is the prior variance of the kernel, a multiplicative constant that determines the magnitude of the kernel. We assume~$k(\mb{x}, \mb{x}) = 1$ for all $\mb{x} \in \mathcal{D}$ without loss of generality, as any multiplicative scaling can be absorbed by the norm bound~$B$.}

In summary, our goal is to efficiently solve~\cref{eq:optimize_f} via a BO algorithm with sublinear regret, where~$f$ lies in some RKHS~$\mathcal{H}_\theta$, but neither the hyperparameters~$\theta$ nor the norm-bound~$\|f\|_{k_\theta}$ are known.

%!TEX root = ../root.tex

\section{Background}
\label{sec:background}

In this section, we review Gaussian processes (GPs) and Bayesian optimization (BO).

\subsection{Gaussian processes (GP)}
\label{sec:gaussian_process}

Based on the assumptions in~\cref{sec:problem_statement}, we can use GPs to infer confidence intervals on~$f$.
The goal of GP inference is to infer a posterior distribution over the nonlinear map~${f(\mb{x}): D \to \mathbb{R}}$ from an input vector~${\mb{x} \in D }$ to the function value~$f(\mb{x})$. This is accomplished by assuming that the function values $f(\mb{x})$, associated with different values of $\mb{x}$, are random variables and that any finite number of these random variables have a joint Gaussian distribution~\citep{Rasmussen2006Gaussian}.
A GP distribution is parameterized by a prior mean function and a covariance function or kernel $k(\mb{x}, \mb{x}')$, which defines the covariance of any two function values~$f(\mb{x})$ and $f(\mb{x}')$ for ${\mb{x}, \mb{x}' \in D}$. In this work, the mean is assumed to be zero without loss of generality. The choice of kernel function is problem-dependent and encodes assumptions about the unknown function.

We can condition a~$GP(0, k(\mb{x}, \mb{x}'))$ on a set of~$t$ past observations ${ \mb{y}_t = (y_1, \dots, y_t) }$ at inputs~$\mathcal{A}_t = \{ \mb{x}_1, \dots, \mb{x}_t \}$ in order to obtain a posterior distribution on~$f(\mb{x})$ for any input~${ \mb{x} \in D }$. The GP model assumes that observations are noisy measurements of the true function value,~$y_t = f(\mb{x}_t) + \omega_t$, where~${\omega_t \sim \mathcal{N}(0,\sigma^2)}$. The posterior distribution is again a $GP(\mu_t(\mb{x}), k_t(\mb{x}, \mb{x}'))$ with mean~$\mu_t$, covariance~$k_t$, and variance~$\sigma_t$, where
\begin{align}
\mu_t(\mb{x}) &= \mb{k}_t(\mb{x})  (\mb{K}_t + \mb{I} \sigma^2)^{-1} \mb{y}_t ,
\label{eq:gp_prediction_mean} \\
k_t(\mb{x}, \mb{x}') &= k(\mb{x}, \mb{x}') - \mb{k}_t(\mb{x}) (\mb{K}_t + \mb{I} \sigma^2)^{-1} \mb{k}_t^\T(\mb{x}'),
\label{eq:gp_prediction_covariance} \\
\sigma^2_t(\mb{x}) &= k_t(\mb{x}, \mb{x}).
\label{eq:gp_prediction_variance}
\end{align}
The covariance matrix~${\mb{K}_t \in \mathbb{R}^{t \times t}}$ has entries ${[\mb{K}_t]_{(i,j)} = k(\mb{x}_i, \mb{x}_j)}$, ${i,j\in\{1,\dots,t\}}$, and
the vector
${\mb{k}_t(\mb{x}) =
\left[ \begin{matrix}
	k(\mb{x},\mb{x}_1),\dots,k(\mb{x},\mb{x}_t)
\end{matrix}  \right]}$
contains the covariances between the input~$\mb{x}$ and the observed data points in~$\mathcal{A}_t$.
The identity matrix is denoted by~${ \mb{I}_t \in \mathbb{R}^{t \times t} }$.

\subsection{Learning RKHS functions with GPs}

The GP framework uses a statistical model that makes different assumptions from the ones made about~$f$ in~\cref{sec:problem_statement}. In particular, we assume a different noise model, and samples from a GP$(0, k(\mb{x}, \mb{x}'))$ are rougher than RKHS funtions and are not contained in~$\mathcal{H}_k$. However, GPs and RKHS functions are closely related \citep{Kanagawa2018Gaussian} and it is possible to use GP models to infer reliable confidence intervals on~$f$ in~\cref{eq:optimize_f}.
\begin{restatable}[\citet{Abbasi-Yadkori2012Online,Chowdhury2017Kernelized}]{lemma}{confidencethm}
Assume that $f$ has bounded RKHS norm $\|f\|_k \leq B$ and that measurements are corrupted by~$\sigma$-sub-Gaussian noise. If $\beta_t^{1/2} = B + 4 \sigma \sqrt{ I(\mb{y}_{t}; f) + 1 + \mathrm{ln}(1 / \delta)}$, then for all~${\mb{x} \in D}$ and~${t \geq 0}$ it holds jointly with probability at least~${1 - \delta}$ that
$
\left|\, f(\mb{x}) - \mu_{t}(\mb{x}) \,\right| \leq \beta_{t}^{1/2} \sigma_{t}(\mb{x}).
$
\label{thm:confidence_interval}
\end{restatable}
\cref{thm:confidence_interval} implies that, with high probability, the true function~$f$ is contained in the confidence intervals induced by the posterior GP distribution that uses the kernel~$k$ from~\cref{thm:confidence_interval} as a covariance function, scaled by an appropriate factor~$\beta_t$. Here, $I(\mb{y}_t; f)$ denotes the mutual information between the GP prior on~$f$ and the~${t}$ measurements~$\mb{y}_{t}$. Intriguingly, for GP models this quantity only depends on the inputs~$\mb{x}_t$ and not the corresponding measurement~$y_t$. Specifically, for a given set of measurements~$\mb{y}_\mathcal{A}$ at inputs~$\mb{x} \in \mathcal{A}$, the mutual information is given by
\begin{equation}
  I(\mb{y}_\mathcal{A}; f) = 0.5 \log | \mb{I} + \sigma^{-2} \mb{K}_{\mathcal{A}} | ,
  \label{eq:mutual_information}
\end{equation}
where~$\mb{K}_\mathcal{A}$ is the kernel matrix~$[k(\mb{x}, \mb{x}')]_{\mb{x}, \mb{x}' \in \mathcal{A}}$ and $|\cdot|$ is the determinant.
%This quantity can be computed cheaply given the Cholesky decomposition of~$\mb{K}_t + \mb{I} \sigma^2$, which is needed to compute~\cref{eq:gp_prediction_mean} in any case.
Intuitively, the mutual information measures how informative the collected samples~$\mb{y}_\mathcal{A}$ are about the function~$f$. If the function values are independent of each other under the GP prior, they will provide large amounts of new information. However, if measurements are taken close to each other as measured by the kernel, they are correlated under the GP prior and provide less information.

\subsection{Bayesian Optimization (BO)}
\label{sec:bayesian_optimization}

BO aims to find the global maximum of an unknown function~\citep{Mockus2012Bayesian}. The framework assumes that evaluating the function is expensive in terms of time required or monetary costs, while other computational resources are comparatively inexpensive.
In general, BO methods model the objective function~$f$ with a statistical model and use it to determine informative sample locations. A popular approach is to model the underlying function with a GP, see~\cref{sec:gaussian_process}.
GP-based BO methods use the posterior mean and variance predictions in~\cref{eq:gp_prediction_mean,eq:gp_prediction_variance} to compute the next sample location.

One commonly used algorithm is the~\textsc{GP-UCB} algorithm by~\cite{Srinivas2012Gaussian}. It uses confidence intervals on the function~$f$, e.g., from~\cref{thm:confidence_interval}, in order to select as next input the point with the largest plasuble function value according to the model,
\begin{equation}
\mb{x}_{t+1} = \underset{\mb{x} \in \mathcal{D}}{\mathrm{argmax}}~ \mu_{t}(\mb{x}) + \beta_t^{1/2} \sigma_{t}(\mb{x}).
\label{eq:gp_ucb}
\end{equation}
Intuitively,~\cref{eq:gp_ucb} selects new evaluation points at locations where the upper bound of the confidence interval of the GP estimate is maximal.
Repeatedly evaluating the function~$f$ at inputs~$\mb{x}_{t+1}$ given by~\cref{eq:gp_ucb} improves the mean estimate of the underlying function and decreases the uncertainty at candidate locations for the maximum, so that the global maximum is provably found eventually~\citep{Srinivas2012Gaussian}.
While~\cref{eq:gp_ucb} is also an optimization problem, it only depends on the GP model of~$f$ and solving it therefore does not require any expensive evaluations of~$f$. %This reflects the assumption of cheap computational resources.

\paragraph{Regret bounds}

\cite{Srinivas2012Gaussian} show that the~\textsc{GP-UCB} algorithm has cumulative regret~$R_t = \bigO(\sqrt{ t \beta_t \gamma_t} )$ for all $t \geq 1$ with the same~$(1-\delta)$ probability as the confidence intervals, e.g., in~\cref{thm:confidence_interval}, hold.
%The scaling factor~$\beta_t$ in~\cref{thm:confidence_interval} is~$\bigO(\gamma_t)$.
Here~$\gamma_t$ is the largest amount of mutual information that could be obtained by any algorithm from at most~$t$ measurements,
\begin{equation}
  \gamma_t = \max_{\mathcal{A} \subset D, \, |\mathcal{A}| \leq t} I(\mb{y}_\mathcal{A}; f).
  \label{eq:gamma_t}
\end{equation}
We refer to~$\gamma_t$ as the \emph{information capacity}, since it can be interpreted as a measure of complexity of the function class associated with a GP prior. It was shown by~\cite{Srinivas2012Gaussian} that~$\gamma_t$ has a sublinear dependence on~$t$ for many commonly used kernels such as the Gaussian kernel. As a result,~$R_t$ has a sublinear dependence on~$t$ so that~$R_t / t \to 0$ and therefore \textsc{GP-UCB} converges to function evaluations close to~$f(\mb{x}^*)$. These regret bounds were extended to Thompson sampling, an algorithm that uses samples from the posterior GP as the acquisition function, by~\cite{Chowdhury2017Kernelized}.

\paragraph{Online hyperparameter estimation}

\begin{figure*}[t]
\centering
\subcaptionbox{Sample from GP prior. \label{fig:mcmc_example_sample}}
    {\includegraphics{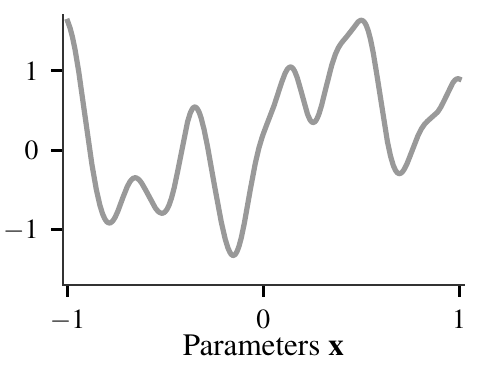}}
\subcaptionbox{GP estimate (RKHS). \label{fig:mcmc_example_map}}
    {\includegraphics{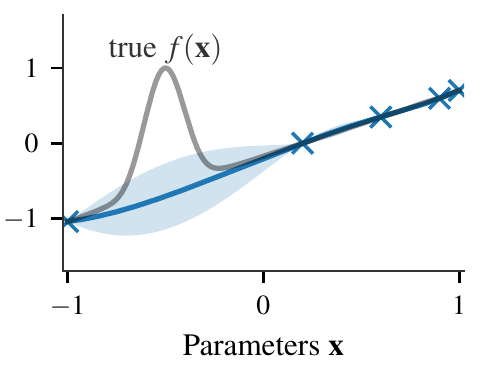}}
\subcaptionbox{Lengthscale distribution. \label{fig:mcmc_example_mcmc}}
    {\includegraphics{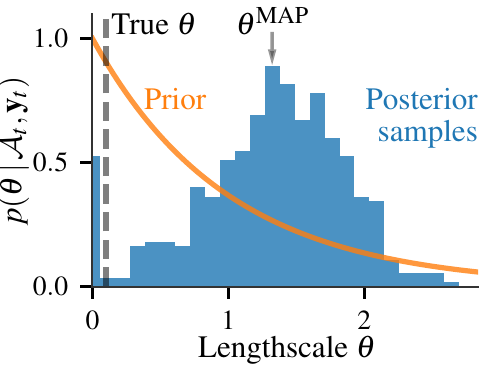}}
\caption{A sample from the GP prior in~\cref{fig:mcmc_example_sample} typically varies at a consistent rate over the input space. However, RKHS functions with the same kernel may be less consistent and can have bumps, as in~\cref{fig:mcmc_example_map} (gray). As a result, inferring the posterior lengthscales based on measurements (blue crosses in~\cref{fig:mcmc_example_map}) can lead to erroneous results. In~\cref{fig:mcmc_example_mcmc}, most of the probability mass of the posterior lengthscales has concentrated around large lengthscales that encode smooth functions. Consequently, the GP's $2\sigma$ confidence intervals in~\cref{fig:mcmc_example_map} (blue shaded) based on the posterior samples do not contain the true function.}
\label{fig:mcmc_example}
\end{figure*}

In the previous section, we have seen that the~\textsc{GP-UCB} algorithm provably converges. However, it requires access to a RKHS norm bound $\|f\|_\theta \leq B$ under the correct kernel hyperparameters~$\theta$ in order to construct reliable confidence intervals using~\cref{thm:confidence_interval}. In practice, these are unknown and have to be estimated online, e.g., based on a prior distribution placed on~$\theta$. Unfortunately, it is well-known that online estimation of the inputs, be it via maximum a posteriori (MAP) or sampling methods, does not always converge to the optimum~\citep{Bull2011Convergence}. The problem does not primarily lie with the inference scheme, but rather with the assumptions made by the GP. In particular, typical samples drawn from a GP with a stationary kernel tend to have a similar rate of change throughout the input space, see~\cref{fig:mcmc_example_sample}. In contrast, the functions inside the RKHS, as specified in~\cref{sec:problem_statement}, can have different rates of change and are thus improbable under the GP prior. For example, the grey function in~\cref{fig:mcmc_example_map} is almost linear but has one bump that defines the global maximum, which makes this function an improbable sample under the GP prior even though it belongs to the RKHS induced by the same kernel. This property of GPs \added{with stationary kernels} means that, for inference, it is sufficient to estimate the lengthscales in a small part of the state-space in order to make statements about the function space globally. This is illustrated in~\cref{fig:mcmc_example_mcmc}, where we show samples from the posterior distribution over the lengthscales based on the measurements obtained from the \textsc{GP-UCB} algorithm in~\cref{fig:mcmc_example_map} (blue crosses). Even though the prior distribution on the lengthscales~$\theta$ is suggestive of short lengthscales, most of the posterior probability mass is concentrated around lengthscales that are significantly larger than the true ones. As a result, even under model averaging over the samples from the posterior distribution of the lengthscales, the GP confidence intervals do not contain the true function in~\cref{fig:mcmc_example_map}. This is not a problem of the inference method applied, but rather a direct consequence of the probabilistic model that we have specified \added{based on the stationary kernel}, which does not consider functions with different rates of change to be likely.

%!TEX root = ../root.tex

\section{The \algname~Algorithm}
\label{sec:theory}

In this section, we extend the \textsc{GP-UCB} algorithm to the case where neither the norm bound~$B$ nor the lengthscales~$\theta$ are known.
In this case, it is always possible that the local optimum is defined by a local bump based on a kernel with small lengthscales, which has not been encountered by the data points as in~\cref{fig:mcmc_example_map}. The only solution to avoid this problem is to keep exploring to eventually cover the input space $\mathcal{D}$ \citep{Bull2011Convergence}. We consider expanding the function space associated with the hyperparameters slowly over time, so that we obtain sublinear regret once the true function class has been identified. Intuitively, this can help BO algorithms avoid premature convergence to local optima caused by misspecified hyperparameters~$\theta$ and~$B$. For example, in~\cref{fig:bo_example_1}, the \textsc{GP-UCB} algorithm has converged to a local maximum. By decreasing the lengthscales, we increase the underlying function class, which means that the GP confidence intervals on the function increase. This enables \textsc{GP-UCB} to explore further so that the global optimum is found, as shown in~\cref{fig:bo_example_3}.

\begin{figure*}[t]
\centering
\subcaptionbox{Stuck in local optimum. \label{fig:bo_example_1}}
    {\includegraphics{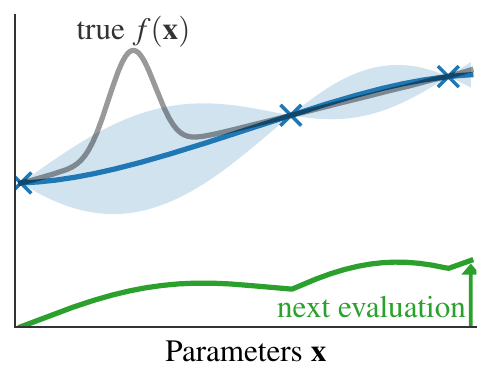}}
\subcaptionbox{Expanding the function class. \label{fig:bo_example_2}}
    {\includegraphics{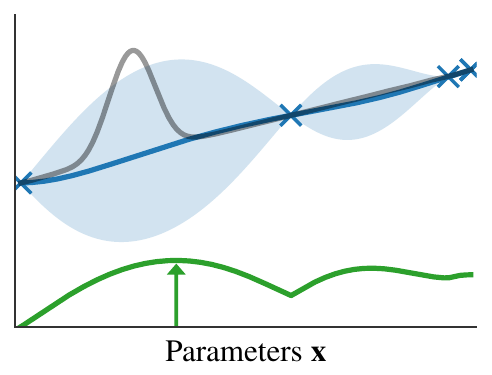}}
\subcaptionbox{Global optimum found. \label{fig:bo_example_3}}
    {\includegraphics{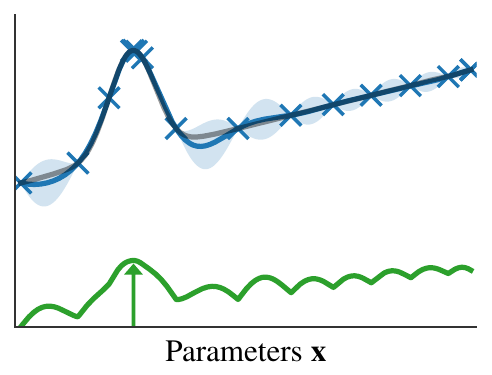}}
\caption{BO algorithms get stuck in local optima when the hyperpararameters of the model are misspecified. In~\cref{fig:bo_example_1}, the true function is not contained within the GP's confidence intervals (blue shaded), so that~$\textsc{GP-UCB}$ only collects data at the local optimum on the right (green arrow), see also~\cref{fig:mcmc_example}. Our method expands the function class over time by scaling the hyperparameters, which encourages additional exploration in~\cref{fig:bo_example_2}. The function class grows slowly enough, so that the global optimum is provably found in~\cref{fig:bo_example_3}.}
\label{fig:bo_example}
\end{figure*}

Specifically, we start with an initial guess~$\theta_0$ and~$B_0$ for the lengthscales and norm bound on~$f$, respectively. Over the iterations, we scale down the lengthscales and scale up the norm bound,
\begin{equation}
  \theta_{t} = \frac{1}{g(t)} \, \theta_0, \qquad B_t = b(t) g(t)^d \, B_0,
  \label{eq:temporal_lengthscales_and_norm}
\end{equation}
where~$g \colon \mathbb{N} \to \mathbb{R}_{> 0}$ and~$b \colon \mathbb{N} \to \mathbb{R}_{> 0}$ with~$b(0)=g(0)=1$ are functions that can additionally depend on the data collected up to iteration~$t$, $\mathcal{A}_t$ and~$\mb{y}_t$. As~$g(t)$ increases, the lengthscales~$\theta_t$ of the kernel become shorter, which enlarges the underlying function space:
\begin{lemma}{\cite[Lemma 4]{Bull2011Convergence}}
  If $f \in \mathcal{H}_\theta$, then $f \in \mathcal{H}_{\theta'}$ for all $0 < \theta' \leq \theta$, and
  \begin{equation}
    \| f \|^2_{\mathcal{H}_{\theta'}} \leq \left( \prod_{i=1}^d  \frac{[\theta]_i}{[\theta']_i} \right) \| f \|^2_{\mathcal{H}_{\theta}} \,.
    \label{eq:rkhs_norm_change}
  \end{equation}
  \label{thm:rkhs_norm_change}
\end{lemma}
\cref{thm:rkhs_norm_change} states that when decreasing the lengthscales~$\theta$, the resulting function space contains the previous one. Thus, as~$g(t)$ increases we consider larger RKHS spaces as candidate spaces for the function~$f$. In addition, as we increase~$b(t)$, we consider larger norm balls within the function space~$\mathcal{H}_{\theta_t}$, which corresponds to more complex functions. However, it follows from \cref{eq:rkhs_norm_change} that, as we increase~$g(t)$, we also increase the norm of any existing function in~$\mathcal{H}_{\theta_0}$ by at most a factor of~$g(t)^d$. This is illustrated in~\cref{fig:norm_balls}: as we scale up the norm ball to~$b(t)B_0$, we capture~$f$ under the initial lengthscales~$\theta_0$. However, by shortening the lengthscales by~$g(t)$, the function~$f$ has a larger norm in the new function space~$\mathcal{H}_{\theta_t} = \mathcal{H}_{\theta_0 / g(t)}$. We account for this through the additional scaling factor~$g(t)^d$ in the norm bound~$B_t$ in~\cref{eq:temporal_lengthscales_and_norm}.

\begin{figure*}[t]
\centering
\subcaptionbox{Scaling of the norm bound. \label{fig:norm_balls}}
    {\includegraphics{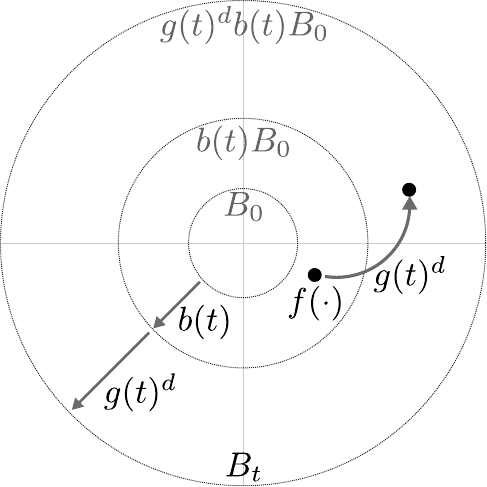}}
\hfill
\subcaptionbox{Cumulative regret with scaling. \label{fig:cumulative_regret}}
    {\includegraphics{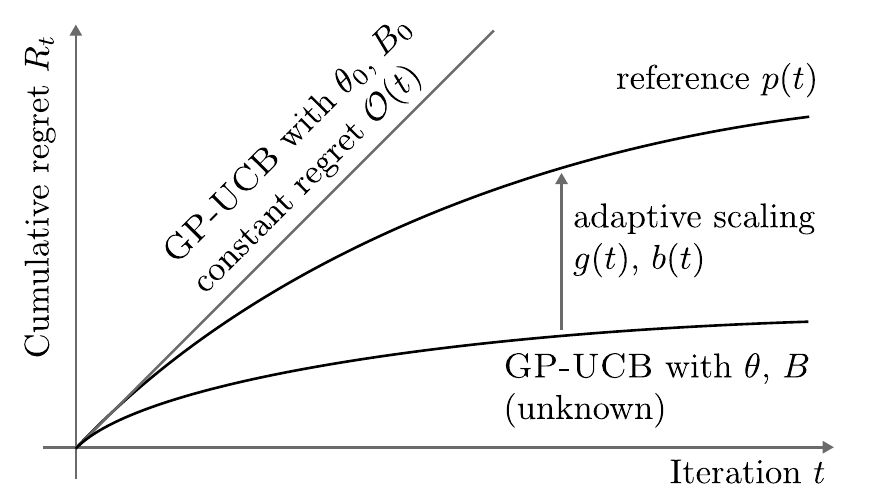}}
\caption{The function $f$ in~\cref{fig:norm_balls} has RKHS norm $\|f\|_{\theta_0} > B_0$. To account for this, we expand the norm ball by~$b(t)$ over time. When we scale down the lengthscales by~$g(t)$, the norm of~$f$ in the resulting RKHS is larger, see~\cref{thm:rkhs_norm_change}. We account for this when defining the norm ball~$B_t$ in~\cref{eq:temporal_lengthscales_and_norm}. In~\cref{fig:cumulative_regret}, the \textsc{GP-UCB} algorithm based on the misspecified hyperparameters $B_0$ and $\theta_0$ does not converge (constant regret). Our method scales the lengthscales and norm bound by $g(t)$ and $b(t)$, so that we eventually capture the true model. Scaling the hyperparameters beyond the true ones leads to additional exploration and thus larger cumulative regret than~\textsc{GP-UCB} with the true, unknown hyperparameters $\theta$ and $B$. However, as long as the cumulative regret is upper bounded by a sublinear function~$p$, ultimately the~\algabbrv~ algorithm converges to the global optimum.}
\label{fig:algorithm_intuition}
\end{figure*}

\paragraph{Theoretical analysis}

Based on the previous derivations together with~\cref{thm:rkhs_norm_change}, it is clear that, if~$g(t)$ and~$b(t)$ are monotonically increasing functions and~$f \in \mathcal{H}_{\theta_{t^*}}$ with~$\|f\|_{\theta_{t^*}} \leq B_{t^*}$ for some~$t^* > 0$, then~$f \in \mathcal{H}_{\theta_{t}}$ and~$\|f\|_{\theta_{t}} \leq B_{t}$ for all~$t \geq t^*$. That is, once the function~$f$ is contained within the norm ball of~$B_{t^*}$ for the lengthscales~$\theta_{t^*}$, then, for any further increase in~$b(t)$ or~$g(t)$, the function~$f$ is still contained in the candidate space~$\{f \in \mathcal{H}_{\theta_t} \,|\, f \leq B_t\}$. Based on this insight, we propose~\textsc{A-GP-UCB} in~\cref{alg:a_gp_ucb}. At iteration~$t$, \textsc{A-GP-UCB} sets the GP lengthscales to~$\theta_t$ and selects new inputs~$\mb{x}_{t+1}$ similar to the~\textsc{GP-UCB} algorithm, but based on the norm bound~$B_t$. We extend the analysis of~\textsc{GP-UCB} and~\cref{thm:confidence_interval} to obtain our main result.
\begin{theorem}
  Assume that~$f$ has bounded  RKHS norm~$ \| f \|_{k_\theta}^2 \leq B $ in a RKHS that is parametrized by a stationary kernel $k_\theta(\mb{x}, \mb{x}')$ with unknown lengthscales $\theta$. Based on an initial guess, $\theta_0$ and~$B_0$, define monotonically increasing functions~$g(t)>0$ and~$b(t)>0$ and run~\textsc{A-GP-UCB} with $\beta_t^{1/2} = b(t) g(t)^d B_0 + 4 \sigma \sqrt{ I_{\theta_t}(\mb{y}_t; f) + 1 + \mathrm{ln}(1 / \delta)}$ and GP lengthscales~$\theta_t = \theta_0 / g(t)$. Then, with probability at least~$(1-\delta)$, we obtain a regret bound of
\begin{equation}
   R_t \leq 2 B \max\left( g^{-1}\left(\max_i \frac{[\theta_0]_i} {[\theta]_i} \right),\, b^{-1}\left( \frac{B}{B_0} \right) \right)
  +\sqrt{ C_1 t \beta_t I_{\theta_t}(\mb{y}_{t}; f) } ,
  \label{eq:thm:regret}
\end{equation}
where~$I_{\theta_t}$ is the mutual information in~\cref{eq:mutual_information} based on the GP model with lengthscales~$\theta_t$ and~$C_1 = 8 / \log(1 + \sigma^{-2} )$.
  \label{thm:main}
\end{theorem}

The proof is given in the appendix. Intuitively, the regret bound in~\cref{eq:thm:regret} splits the run of the algorithm into two distinct phases. In the first one, either the RKHS space~$\mathcal{H}_{\theta_t}(\mathcal{D})$ or the norm bound~$B_t$ are too small to contain the true function~$f$. Thus, the GP confidence intervals scaled by~$\beta_t^{1/2}$ do not necessarily contain the true function~$f$, as in~\cref{fig:mcmc_example_map}. In these iterations, we obtain constant regret that is bounded by~$2B$, since~$\|f\|_\infty \leq \|f\|_{\theta} \leq B$. After both~$g$ and~$b$ have grown sufficiently in order for the considered function space to contain the true function, the confidence bounds are reliable and we can apply the theoretical results of the~\textsc{GP-UCB} algorithm. This is illustrated in~\cref{fig:cumulative_regret}: If the initial hyperparameters~$\theta_0$ and~$B_0$ are misspecified, the confidence intervals do not contain~$f$ and \textsc{GP-UCB} does not converge. We avoid this problem by increasing~$b(t)$ and~$g(t)$ over time, so that we eventually contain $f$ in our function class. However, increasing the norm ball and decreasing the lengthscales beyond the true ones causes additional exploration and thus additional cumulative regret relative to \textsc{GP-UCB} with the true, unknown hyperparameters. This additional regret represents the cost of not knowing the hyperparameters in advance. As long as the overall regret remains bounded by a sublinear function~$p(t)$, our method eventually converges to the global optimum.
The regret bound in~\cref{eq:thm:regret} depends on the true hyperparameters~$\theta$ and~$B$. However, the algorithm does not depend on them. \cref{thm:main} provides an instance-specific bound, since the mutual information depends on the inputs in~$\mathcal{A}_t$. One can obtain a worst-case upper bound by bounding~$I_{\theta_t}(\mb{y}_t; f) \leq \gamma_t(\theta_t)$, which is the worst-case mutual information as in~\cref{eq:gamma_t}, but based on the GP model with lengthscales~$\theta_t$.  While \cref{thm:main} assumes that the noise properties are known, the results can be extended to estimate the noise similar to~\citet{Durand2018Streaming}.

\begin{algorithm}[t]
  \caption{\algname (\algabbrv)}
  \begin{algorithmic}[1]
    \STATE{} \textbf{Input:} Input space~$\mathcal{D}$, $GP(0, k(\mb{x}, \mb{x}'))$, functions~$g(t)$ and~$b(t)$ \\
    \STATE{} Set $B_0 = 1$ and $\theta_0 = \mathrm{diam}(\mathcal{D})$
    \FORALL{$t = 0, 1, 2, \dots$}
      \STATE{} Set the GP kernel lengthscsales to~$\theta_{t} = \theta_0 / g(t)$
      \STATE{} $\beta_t^{1/2} \gets B(t) + 4 \sigma \sqrt{ I_{\theta_t}(\mb{y}_t; f) + 1 + \mathrm{ln}(1 / \delta)}$ with $B(t) = b(t) g(t)^d B_0$
      \STATE{} Choose $\mb{x}_{t+1} = \argmax_{\mb{x} \in \mathcal{D}} \, \mu_{t}(\mb{x}) + \beta_t^{1/2} \sigma_{t}(\mb{x})$
      \STATE{} Evaluate $y_{t+1} = f(\mb{x}_{t+1}) + \epsilon_{t+1}$
      \STATE{} Perform Bayesian update to obtain~$\mu_{t+1}$ and~$\sigma_{t+1}$
    \ENDFOR{}
  \end{algorithmic}
\label{alg:a_gp_ucb}
\end{algorithm}

For arbitrary functions~$g(t)$ and~$b(t)$, the candidate function space $\{f \in \mathcal{H}_{\theta_t} \,|\, f \leq B_t\}$ can grow at a faster rate than it contracts by selecting informative measurements~$y_t$ according to~\cref{eq:gp_ucb}. In particular, in the regret term~$\sqrt{C_1 t \beta_t \gamma_t}$ both~$\beta_t$ and~$\gamma_t$ depend on the scaling factors~$g(t)$ and~$b(t)$. If these factors grow at a faster rate than~$\sqrt{t}$, the resulting algorithm does not enjoy sublinear regret. We have the following result that explicitly states the dependence of~$\gamma_t$ on the scaling factor~$g(t)$.
\begin{proposition}
  Let~$k_\theta$ be a stationary kernel parameterized by lengthscales~$\theta$ as in~\cref{eq:stationary_lengthscale_kernel} and define~$\gamma_t(\theta)$ for lengthscales~$\theta$ as in~\cref{eq:gamma_t}. Define the lengthscales as~$\theta_t = \theta_0 / g(t)$ as in~\cref{eq:temporal_lengthscales_and_norm}.
  \begin{itemize}
    \item If~$k(\mb{x}, \mb{x}')=\mathrm{exp}(-\frac{1}{2}\|\mb{x} - \mb{x}'\|_2^2)$ is the squared exponential (Gaussian) kernel, then \begin{equation}
      \gamma_t(\theta_t) = \bigO \left( g(t)^{d}(\log t)^{d+1} \right)
      \label{eq:gamma_t_gaussian}
    \end{equation}
    \item If~$k(\mb{x}, \mb{x}') = (2^{1-\nu} /\, \Gamma(\nu))\, r^\nu B_\nu(r)$ is the Mat\'ern kernel, where $r = \sqrt{2 \nu} \|\mb{x} - \mb{x}'\|_2$, $B_\nu$ is the modified Bessel function with $\nu > 1$, and~$\Gamma$ is the gamma function. Then
    \begin{equation}
      \gamma_t(\theta_t) = \bigO \left( g(t)^{2\nu + d} t^{\frac{d(d+1)}{ 2\nu + d(d + 1) }} \log t \right)
      \label{eq:gamma_t_matern}
    \end{equation}
  \end{itemize}
  \label{thm:gamma_t_lengthscale_bounds}
\end{proposition}
Proposition~\ref{thm:gamma_t_lengthscale_bounds} explicitly states the relationship between~$\gamma_t$ and~$g(t)$. For the Gaussian kernel, if we scale down the lengthscales by a factor of two, the amount of mutual information that we can gather in the worst case,~$\gamma_t$, grows by~$2^d$. Given the dependence of~$\gamma_t$ on~$g(t)$, we can refine~\cref{thm:main} to obtain concrete regret bounds for two commonly used kernels.
\begin{corollary}
  If, under the assumptions of~\cref{thm:main},~$g(t)$ and~$b(t)$ grow unbounded, then we obtain the following, high-probability regret bounds for~\cref{alg:a_gp_ucb}:
  \begin{itemize}
    \item Squared exponential kernel: $R_t \leq \bigO \left( b(t) \sqrt{t g(t)^{3d} \gamma_t(\theta_0)} + g(t)^d \gamma_t(\theta_0) \sqrt{t} \right)$;
    \item Mat\'ern kernel: $R_t \leq \bigO \left( b(t) \sqrt{t g(t)^{2 \nu + 3 d} \gamma_t(\theta_0)} + g(t)^{\nu + d} \gamma_t(\theta_0)\sqrt{t} \right)$.
  \end{itemize}
  \label{cor:concrete_regret_bounds}
\end{corollary}
If~$b(t)$ and~$g(t)$ grow unbounded, the first term of the cumulative regret in~\cref{eq:thm:regret} can be upper bounded by a constant. The remaining result is obtained by plugging in~$\beta_t$ and the bounds from~\cref{eq:gamma_t}. Thus, any functions $g(t)$ and $b(t)$ that render the regret bounds in Corollary~\ref{cor:concrete_regret_bounds} sublinear allow the algorithm to converge, even though the true lengthscales and norm bound are unknown.

\added{
The specific choices of~$b(t)$ and~$g(t)$ matter for the regret bound in \cref{thm:main} in practice. Consider the one-dimensional case~$d=1$ for the Gaussian kernel. Given the true hyperparameters~$B$ and~$\theta$, if we set~$g(t)= \theta_0 / \theta$ and~$b(t) = B / B_0$ to be constant, we recover the non-adaptive regret bounds of~\textsc{GP-UCB} with known hyperparameters. If~$g(t)$ depends on~$t$ and grows slowly, then the algorithm incurs constant regret during the initial rounds when the model is misspecified, while functions~$g$ that grow to values larger than the optimal ones lead to additional exploration and incur an additional~$\bigO(b(t)g(t)^{3d/2)})$ factor in the cumulative regret in later rounds, as in Corollary~\ref{cor:concrete_regret_bounds}. In the following section, we discuss appropriate choices for these functions in practice.
}

%!TEX root = ../root.tex

\subsection{Choosing the scaling functions~$g(t)$ and~$b(t)$}
\label{sec:practical}

It follows from~\cref{thm:main} that \textsc{A-GP-UCB} achieves no-regret for any functions~$g(t)$ and~$b(t)$ that increase without bound and render~\cref{eq:thm:regret} sublinear in~$t$. Thus, the corresponding BO routine converges to the optimal value eventually. For example,~$b(t)=g(t)=\log(t)$ satisfy this condition. However, the convergence guarantees in~\cref{thm:main} are only meaningful once $t$ has grown sufficiently so that the true function is contained in the confidence intervals. In practice, BO is often used with objective functions~$f$ that are expensive to evaluate, which imposes a hard constraint on the number of evaluations. For the regret bounds to be meaningful in this setting, we must choose functions~$g$ and~$b$ that grow fast enough to ensure that the constant regret period in~\cref{eq:thm:regret} is small, yet slow enough that the effect of the sublinear regret is visible for small enough~$t$. In the following, we propose two methods to choose~$g(t)$ and~$b(t)$  \emph{adaptively}, based on the observations seen so far.

For convenience, we fix the relative magnitude of~$g(t)$ and~$b(t)$. In particular, we define $b(t) = 1 + \epsilon_b(t)$ and $g(t)^d = 1 + \epsilon_g(t)$ together with a weighting factor~$\lambda = \epsilon_b(t) / \epsilon_g(t)$ that encodes whether we prefer to scale up the norm bound using~$b(t)$ or decrease the lengthscales using~$g(t)$. This allows us to reason about the overall magnitude of the scaling~$h(t) = (1 + \epsilon_g(t))(1 + \epsilon_b(t)) \geq 1$, which can be uniquely decomposed into~$g(t)$ and~$b(t)$ given~$\lambda$. For~$\lambda=0$ we have~$g(t) = h(t)$, $b(t)=1$ and the algorithm prefers to attribute an increase in~$h(t)$ to~$g(t)$ and shorten the lengthscales, while for~$\lambda \to \infty$ the algorithm prefers to scale up the RKHS norm. The assumptions in Corollary~\ref{cor:concrete_regret_bounds} hold for any~$\lambda \in (0, \infty)$ if~$h(t)$ grows unbounded. Moreover, we have that~$g(t)^d \leq h(t)$ and~$b(t) \leq h(t)$.

\paragraph{Reference regret}

While any function~$h(t)$ that grows unbounded and renders the cumulative regret in~\cref{thm:main} sublinear makes our method to converge to the optimum eventually, we want to ensure that our method performs well in finite time too. For fixed hyperparameters with~$h(t)=1$, which implies $g(t)=b(t)=1$, our algorithm reduces to~\textsc{GP-UCB} with hyperparameters~$\theta_0$ and $B_0$ and the regret bound term~$\sqrt{C_1 \beta_t \gamma_t(\theta_0)}$ is sublinear, which is illustrated by the bottom curve in~\cref{fig:cumulative_regret}. However, this does not imply no-regret if hyperparameters are misspecified as in~\cref{fig:bo_example_1}, since the first term in~\cref{thm:main} is unbounded in this case. To avoid this, we must increase the scaling factor~$h(t)$ to consider larger function classes.

We propose to define a sublinear reference regret~$p(t)$, see~\cref{fig:cumulative_regret}, and to scale~$h(t)$ to match an estimate of the regret with respect to the current hyperparameters to this reference. As \textsc{GP-UCB} converges, the regret estimate with respect to the current hyperparameters levels off and drops below the reference~$p(t)$. In these cases, we increase~$h(t)$ to consider larger function classes and explore further. The choice of~$p(t)$ thus directly specifies the amount of additional regret one is willing to incur for exploration. Specifically, given a regret estimate~$\bar{R}_t(h)$ that depends on the data collected so far and the selected scaling~$h$, we obtain~$h(t)$ from matching the reference, $\bar{R}_t(h) = p(t)$, as
\begin{equation}
  h^*(t) = \bar{R}_t^{-1}(p(t)), \qquad h(t) = \max (h^*(t), \, h(t-1)) .
  \label{eq:h_opt_prob}
\end{equation}
Here we explicitly enforce that~$h(t)$ must be an increasing function. In the following, we consider estimators~$\bar{R}_{t}$ that are increasing functions of~$h$, so that~\cref{eq:h_opt_prob} can be solved efficiently via a line search.

Whether choosing~$h(t)$ according to~\cref{eq:h_opt_prob} leads to a sublinear function depends on the regret estimator~$\bar{R}_t$. However, it is always possible to upper bound the~$h(t)$ obtained from~\cref{eq:h_opt_prob} by a fixed sublinear function. This guarantees sublinear regret eventually. In the following, we consider two estimators that upper bound the cumulative regret experienced so far with respect to the hyperparameters suggested by~$h(t)$.

\paragraph{Regret bound}
As a first estimator for the cumulative regret, we consider the regret bound on~$R_t$ in~\cref{eq:thm:regret}. We focus on the Gaussian kernel, but the arguments transfer directly to the case of the Mat\'ern kernel. The term~$\sqrt{C_1 t \,\beta_t\, I_{\theta_{t}}(\mb{y}_t; f) }$ bounds the regret with respect to the current function class specified by~$\theta_t$. In addition to the direct dependence on~$b(t)g(t)^d$ in~$\beta_t$, the regret bound also depends on~$g(t)$ implicitly through the mutual information~$I_{\theta_t}(\mb{y}_t; f)$, where $\theta_t = \theta_0 / g(t)$. To make the dependence on~$g(t)$ more explicit, we use~\cref{thm:gamma_t_lengthscale_bounds} and rewrite the mutual information as~$(g(t) / g(t-1))^d I_{\theta_{t-1}}(\mb{y}_t; f)$ instead. Note that the scaling factor was derived for~$\gamma_t$, but remains a good indicator of increase in mutual information in practice. With this replacement we use
\begin{equation}
\bar{R}_t(h) = \sqrt{C_1 t \,\beta_t\left(b(t), g(t)\right)\,  g(t)^d I_{\theta_{t-1}}(\mb{y}_t; f) }
\label{eq:gt_scaleinfo}
\end{equation}
to estimate the regret, where the term~$\beta_t(b, g)$ is as in~\cref{thm:main}, but with the mutual information similarly replaced with the explicit dependence on~$g(t)$. Solving~\cref{eq:h_opt_prob} together with~\cref{eq:gt_scaleinfo} is computationally efficient, since computing~$\bar{R}_t$ does not require inverting the kernel matrix.

\added{
\paragraph{One step predictions}

While~\cref{eq:gt_scaleinfo} is fast to compute, it requires us to know the dependence of~$\gamma_t(\theta_t)$ on~$h(t)$. Deriving analytic bounds can be infeasible for many kernels.
As an alternative, we estimate the regret one-step ahead directly.
In particular, if the considered function class is sufficiently large and our confidence intervals hold at all time steps~$t > 0$, then the one-step ahead cumulative regret~$R_{t+1}$ for our algorithm at iteration~$t$ is bounded from above by
\begin{equation}
  \bar{R}_{t} = 2 \sum_{j=1}^{t} \beta_j^{1/2} \sigma_j(\mb{x}_{j + 1}),
  \label{eq:hyp_opt}
\end{equation}
where each~$\beta_t$ and~$\sigma_t$ is based on the corresponding hyperparameters~$\theta_t$. In \cref{thm:main}, $R_{t+1}$ is further upper-bounded by~\cref{eq:thm:regret}. The regret estimate in~\cref{eq:hyp_opt} depends on~$\mb{x}_{t+1}$, which is the next input that would be evaluated based on the UCB criterion with GP hyperparameters scaled according to~$h(t)$. As the hyperparameters for previous iterations are fixed, the only term that depends on~$h(t)$ is the bound on the instantaneous regret,~$r_t \leq 2 \beta_t \sigma_t(\mb{x}_{t+1})$.
Unlike~\cref{eq:gt_scaleinfo}, \cref{eq:hyp_opt} is not able to exploit the known dependence of~$\gamma_t$ on~$h(t)$, so that it cannot reason about the long-term effects of changing~$h(t)$. This means that, empirically, the cumulative regret may overshoot the reference regret, only to settle below it later.

Scaling~$h(t)$ according to~\cref{eq:hyp_opt} provides an interesting perspective on the method by~\citet{Wang2014Theoretical}. They decrease the kernel lengthscales whenever~$\sigma_t(\mb{x}_{t+1}) \leq \kappa$. In our framework, this corresponds to~$p(t) = \sum_{j=1}^t 2 \beta_j \max(\kappa, \sigma_j(\mb{x}_{j+1})) \geq \kappa t$, which is not sublinear. As a consequence, while they ultimately bound the cumulative regret using the smallest possible lengthscale, the choice for~$p(t)$ forces too much exploration to achieve sublinear regret before the lower bound is reached. In contrast, if we choose~$p(t)$ to be sublinear, then the function class grows slowly enough to ensure more careful exploration. This allows us to achieve sublinear regret in the case when a lower bound on the hyperparameters it not known.
}

\subsection{Practical Considerations and Discussion}
\label{sec:exploration}
\label{sec:discussion}

In this section, we discuss additional practical considerations and show how to combine the theoretical results with online inference of the hyperparameters.

\paragraph{Online inference and exploration strategies}
The theoretical results presented in the previous sections extend to the case where the initial guess~$\theta_0$ of the GP's lengthscale is improved online using any estimator, e.g., with MAP estimation to obtain~$\theta_t^\mathrm{MAP}$. Theoretically, as long as the change in~$\theta_0$ is bounded, the cumulative regret increases by at most a constant factor. In practice, this bound can always be enforced by truncating the estimated hyperparameters. Moreover, the scaling induced by online inference can be considered to be part of~$g(t)$ according to~\cref{thm:rkhs_norm_change}, in which case the norm bound can be adapted accordingly. In practice, online inference improves performance drastically, as it is often difficult to specify an appropriate relative initial scaling of the lengthscales~$\theta_0$.

In more than one dimension,~$d>1$, there are multiple ways that MAP estimation can be combined with the theoretical results of the paper. The simplest one is to enforce an upper bound on the lengthscales based on~$g(t)$,
\begin{equation}
  \theta_t = \min(\theta_t^{\mathrm{MAP}},\, \theta_0 \,/\, g(t)),
  \label{eq:exploration_max}
\end{equation}
\added{where the min is taken elementwise}. This choice is similar to the one by~\citet{Wang2016Bayesian}. If all entries of~$\theta_0$ have the same magnitude, this scaling can be understood as encouraging additional exploration in the smoothest direction of the input space first. This often makes sense, since MAP estimates tend to assume functions that are too smooth, see~\cref{fig:mcmc_example}. However, it can be undesirable in the case when the true function only depends on a subset of the inputs. In these cases, the MAP estimate would correctly eliminate these inputs from the input space by assigning long lengthscales, but the scaling in~\cref{eq:exploration_max} would encourage additional exploration in these directions first. \added{Note that eventually exploring the entire input space is unavoidable to avoid getting stuck in local optima~\citep{Bull2011Convergence}.}

An alternative approach is to instead scale down the MAP estimate directly,
\begin{equation}
  \theta_t = \theta_t^{\mathrm{MAP}}  \, / \, \max(g(t), \, 1).
  \label{eq:exploration_scale}
\end{equation}
This scaling can be understood as evenly encouraging additional exploration in all directions.
While~\cref{eq:exploration_scale} also explores in directions that have been eliminated by the MAP estimate, unlike~\cref{eq:exploration_max} it simultaneously explores all directions relative to the MAP estimate. From a theoretical point of view, the choice of exploration strategy does not matter, as in the limit as~$t \to \infty$ all lengthscales approach zero. In the one-dimensional case, the two strategies are equivalent. Both strategies use the MAP lengthscales for BO in the nominal case, but the~$g(t)$ factor eventually scales down the lengthscales further. This ensures that our method only improves on the empirical performance of BO with MAP estimation.

In practice, maximum likelihood estimates for the inputs are often good enough when the underlying function resembles a sample from a GP. Thus, the approach presented in this paper is most relevant when the underlying function has some `nonstationarity'. In the literature, other approaches to deal with nonstationarity have been proposed. For example,~\cite{Snoek2013Input} scale the input inputs through a beta function and infer its hyperparameters online. Our approach can easily be combined with any such method, as it works on top of any estimate provided by the underlying inference scheme. Moreover, in high-dimensional spaces one can combine our algorithm with methods to automatically identify a low-dimensional subspace of~$\mathcal{D}$~\citep{Djolonga2013HighDimensional,Wang2016Bayesian}.

In this paper, we have considered the kernel to be fixed, and only adapted the lengthscales and norm bound. However, often the kernel structure itself is a critical hyperparameter~\citep{Duvenaud2011Additive}. The strategy presented in this paper could be used to add rougher kernels over time or, for example, to adapt the~$\nu$ input of the Mat\'ern kernel, which determines its roughness.

\paragraph{Confidence intervals}

Empirically, $\beta_t$ is often set to a constant rather than using the theoretical bounds in~\cref{thm:confidence_interval}, which leads to (point-wise) confidence intervals when~$f$ is sampled from a GP model. In particular, typically measurement data is standardized to be zero mean and unit variance and~$\beta_t$ is set to two or three. This often works well in practice, but does not provide any guarantees. However, if one were to believe the resulting confidence bounds, our method can be used to avoid getting stuck in local optima, too. In this case on can set~$h(t) = g(t)$ and apply our method as before.

\paragraph{General discussion}
Knowing how the sample complexity of the underlying BO algorithm depends on the lengthscales also has implications in practice. For example,~\cite{Wang2016Bayesian} and~\cite{Wabersich2016Advancing} suggest to scale down the lengthscales by a factor of~$2$ and roughly~$1.1$, respectively, although not at every iteration. As shown in~\cref{sec:theory}, this scales the regret bound by a factor of~$g^d$, which quickly grows with the number of dimensions. Exponentiating their factors with~$1/d$ is likely to make their approaches more robust when BO is used in high-dimensional input spaces~$\mathcal{D}$.

Lastly, in a comparison of multiple BO algorithms (acquisition functions) on a robotic platform,~\cite{Calandra2014Experimental} conclude that the~\textsc{GP-UCB} algorithm shows the best empirical performance for their problem. They use the theoretical version of the algorithm by~\citet{Srinivas2012Gaussian}, in which~$\beta_t$ grows with an additional factor of~$\bigO(\sqrt{\log(t^2)})$ relative to~\cref{thm:confidence_interval}. In our framework with the bounds in~\cref{thm:confidence_interval}, this is equivalent to scaling up the initial guess for the RKHS norm bound for~$f$ by the same factor at every iteration, which increases the function class considered by the algorithm over time. We conjecture that this increase of the function class over time is probably responsible for pushing the MAP estimate of the lengthscales out of the local minima, which in turn led to better empirical performance.

%!TEX root = ../root.tex

\section{Experiments}
\label{sec:experiments}

In this section, we evaluate our proposed method on several benchmark problems. \added{As  baselines, we consider algorithms based on the \textsc{UCB} acquisition function.} We specify a strong gamma prior that encourages short lengthscales, and consider both maximum a posteriori (MAP) point-estimates of the hyperparameters and a Hamiltonian Monte Carlo (HMC) approach that samples from the posterior distribution of the hyperparameters and marginalizes them out. Unless otherwise specified, the initial lengthscales are set to~$\theta_0 = \mb{1}$, the initial norm bound is~$B_0=2$, the confidence bounds hold with probability at least~$\delta=0.9$, and the tradeoff factor between~$b(t)$ and~$g(t)$ is~$\lambda = 0.1$.

We follow several best-practices in BO to ensure a fair comparison with the baselines. We rescale the input space~$\mathcal{D}$ to the unit hypercube in order to ensure that both the initial lengthscales and the prior over lengthscales are reasonable for different problems. As is common in practice, the comparison baselines use the empirical confidence intervals suggested in~\cref{sec:discussion}, instead of the theoretical bounds in~\cref{thm:confidence_interval} that are used for our method. Lastly, we initialize all GPs with~$2^d$ measurements that are collected uniformly at random within the domain~$\mathcal{D}$.
To measure performance, we use the cumulative regret that has been the main focus of this paper. In addition, we evaluate the different methods in terms of simple regret, which is the regret of the best inputs evaluated so far, $\max_{x\in\mathcal{D}} f(x) - \max_{t' <= t} f(\mb{x}_{t'})$. This metric is relevant when costs during experiments do not matter and BO is only used to determine high-quality inputs by the end of the optimization procedure.

\subsection{Synthetic Experiments}

\paragraph{Example function}

\begin{figure*}[t]
\centering
\subcaptionbox{Simple regret. \label{fig:bumplinear_simregret}}
    {\includegraphics{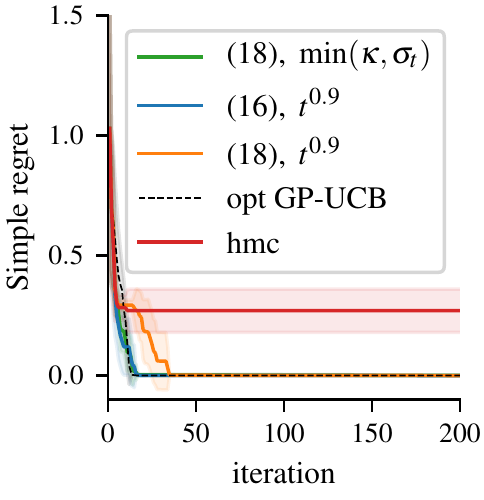}}
\subcaptionbox{Cumulative regret. \label{fig:bumplinear_cumregret}}
    {\includegraphics{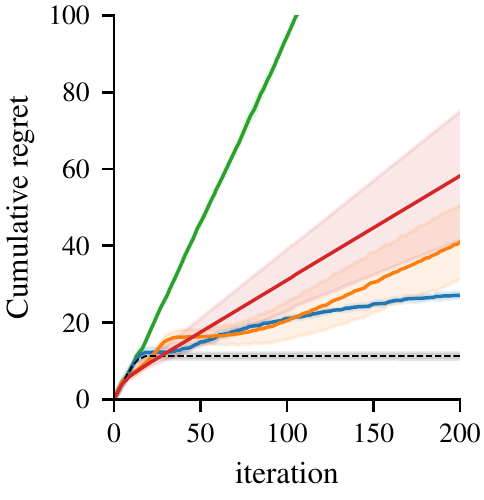}}
\subcaptionbox{Scaling~$g(t)$. \label{fig:bumplinear_funclass}}
    {\includegraphics{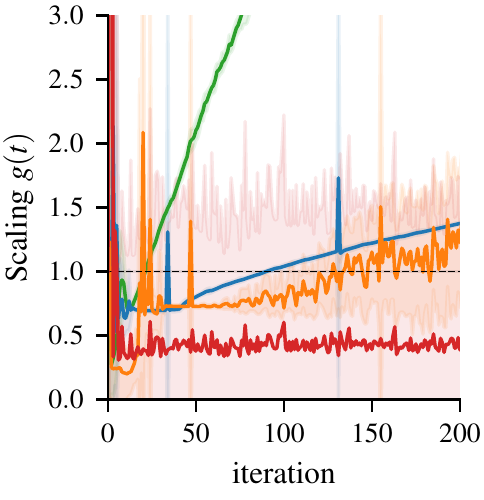}}
\caption{Mean and standard deviation of the empirical simple and cumulative regret over ten different random initializations for the function in~\cref{fig:bo_example}. The HMC baseline (red) gets stuck in a local optimum and obtains constant regret in~\cref{fig:bumplinear_simregret}. \textsc{GP-UCB} with the true hyperparameters (gray dashed) obtains the lowest cumulative regret in~\cref{fig:bumplinear_cumregret}. However, our methods (orange/blue) increase the function class over time, see~\cref{fig:bumplinear_funclass}, and thus obtain sublinear regret without knowing the true hyperparameters.}
\label{fig:bumplinear_results}
\end{figure*}

We first evaluate all proposed methods on the example function in~\cref{fig:bo_example}, which lives inside the~RKHS associated with a Gaussian kernel with~$\theta=0.1$ and has norm~$\|f\|_{k_\theta}=2$. We evaluate our proposed method for the sublinear reference function~$p(t) = t^{0.9}$ together with maximum a posteriori hyperparameter estimation. We compare against both \textsc{GP-UCB} with the fixed, correct hyperparameters and HMC hyperparameter estimation. Additionally, we consider a modified variant of the method suggested by~\cite{Wang2014Theoretical}, see~\cref{sec:practical}. Rather than scaling the lengthscales by a fixed constant, we conduct a line search to find the smallest possible scaling factor that renders~$\sigma_t(\mb{x}_{t+1}) \geq \kappa = 0.1$. This is the most conservative variant of the algorithm. Note that we do not know a lower bound on the hyperparameters and therefore do not enforce it.

The results of the experiments are shown in~\cref{fig:bumplinear_results}. The simple regret plot in~\cref{fig:bumplinear_simregret} shows that all methods based on hyperparameter adaptation evaluate close-to-optimal inputs eventually, and do so almost as quickly as~\textsc{GP-UCB} based on the true hyperparameters (black, dashed). However, the method based on HMC hyperparameter estimation (red) considers functions that are too smooth and gets stuck in local optima, as in~\cref{fig:bo_example}. This can also be seen in~\cref{fig:bumplinear_funclass}, which plots the effective scaling~$g(t)$ based on the combination of Bayesian hyperparameter estimation and hyperparameter adaptation through~$h(t)$. The HMC hyperparameters  consistenly over-estimate the lengthscales by a factor of roughly two. In contrast, while the MAP estimation leads to the wrong hyperparameters initially, the adaptation methods in~\cref{eq:gt_scaleinfo,eq:hyp_opt} slowly increase the function class until the true lengthscales are found eventually. It can be seen that the one step estimate~\cref{eq:hyp_opt} (orange) is more noisy than the upper bound in~\cref{eq:gt_scaleinfo} (blue).

 \begin{figure*}[t]
 \centering
 \subcaptionbox{Simple regret. \label{fig:sample_simregret}}
     {\includegraphics{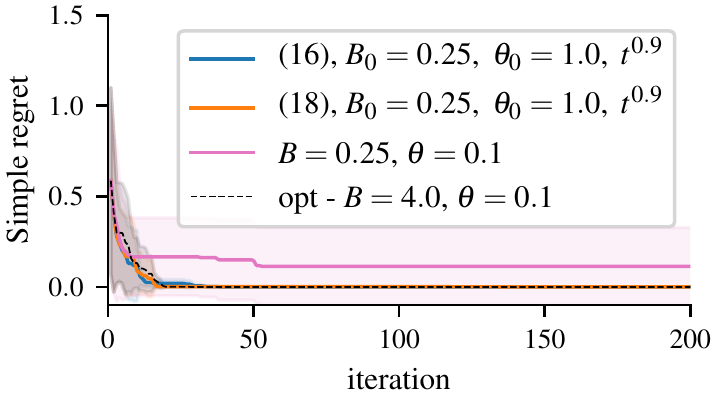}}
 \subcaptionbox{Cumulative regret. \label{fig:sample_cumregret}}
     {\includegraphics{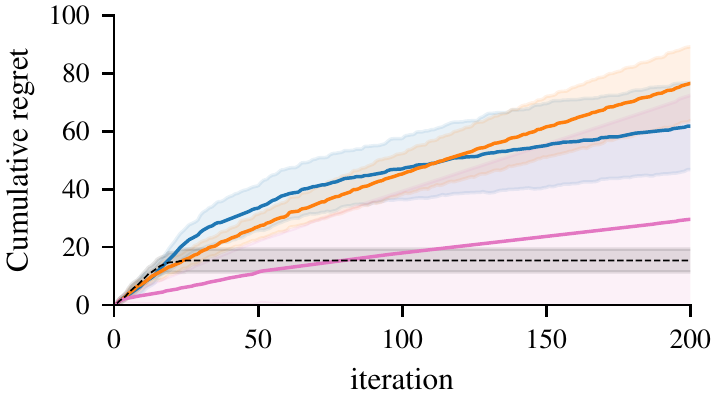}}
 \caption{Simple and cumulative regret over 10 random seeds for samples from a GP with bounded RKHS norm. The \textsc{GP-UCB} algorithm with misspecified hyperparameters (magenta) fails to converge given only a wrong choice of~$B_0$. In contrast, our methods (blue/orange) converge even though~$\theta_0$ is misspecified in addition.}
 \label{fig:sample_results}
 \end{figure*}

While all adaptation methods determine good inputs quickly according to the simple regret, they perform differently in terms of the cumulative regret in~\cref{fig:bumplinear_cumregret}. As expected, the HMC method (red line) converges to a local optimum and experiences constant regret increase equal to the simple regret at every time step. The modified method of~\cite{Wang2014Theoretical} (green line) expands the function class too aggressively and also achieves constant regret. Empirically, their method always explores and never repeatedly evaluates close-to-optimal inputs that would decrease cumulative regret. While the method works well in terms of simple regret, without a lower bound on the hyperparameters it never converges to sublinear regret. As expected from~\cref{thm:main}, GP-UCB based on the optimal hyperparameters achieves the lowest cumulative regret. Our two methods expand the function class over time, which allows them to converge to close-to-optimal inputs, even though MAP estimation estimates the hyperparameters wrongly initially. While the regret is sublinear, the additional exploration caused by~$g(t)$ means that the cumulative regret is larger. This is the additional cost we incur for not knowing the hyperparameters in advance.

\paragraph{Samples from a GP}

As a second experiment, we compare~\textsc{GP-UCB} to \textsc{A-GP-UCB} on samples drawn from a GP when the norm bound~$B_0$ is misspecified. Samples from a GP are not contained in the RKHS. To avoid this technical issue, we sample function values from the posterior GP at only a finite number of discrete gridpoints and interpolate between them using the kernel with the correct lengthscales~$\theta$. We rescale these functions to have~RKHS norm of~$B=4$, but use~$B_0=0.25$ as an initial guess for both BO algorithms and do not use any hyperparameter estimation. Even though we use the correct kernel lengthscales for \textsc{GP-UCB},~$\theta_0 = \theta = 0.1$, this discrepancy means that the true function is not contained in the initial confidence intervals. As before, for our method we use the reference regret~$p(t) = t^{0.9}$ and additionally misspecify the lengthscales,~$\theta_0 = 1$.

The results are shown in~\cref{fig:sample_results}. \textsc{GP-UCB} with the correct hyperparameters (black, dashed) obtains the lowest cumulative regret. However, it fails to converge when hyperparameters are misspecified (magenta), since the confidence intervals are too small to encourage any exploration. In contrast, our methods (blue/orange) converge to close-to-optimal inputs as in the previous example.

\subsection{Logistic Regression Experiment}

\begin{figure*}[t]
    \centering
    \subcaptionbox{Simple regret. \label{fig:logistic_simregret}}
        {\includegraphics{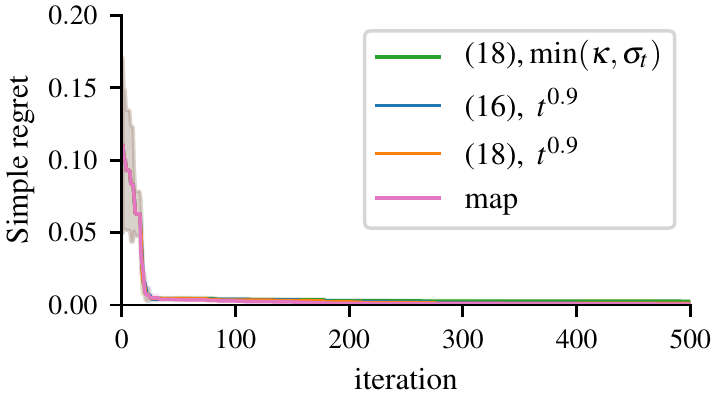}}
    \subcaptionbox{Cumulative regret. \label{fig:logistic_cumregret}}
        {\includegraphics{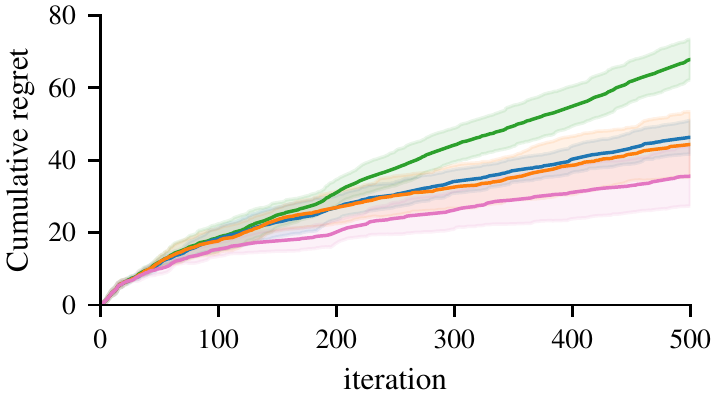}}
    \caption{Simple and cumulative regret over 5 random seeds for a logistic regression problem. All methods determine close-to-optimal parameters. However, our methods explore more to counteract misspecified hyperparameters.}
    \label{fig:logistic_results}
\end{figure*}

Lastly, we use our method to tune a logistic regression problem on the MNIST data set~\citep{Lecun1998MNIST}. As in the experiment in~\cite{Klein2016Bayesian}, we consider four training inputs: the learning rate, the~$l_2$ regularization constant, the batch size, and the dropout rate. We use the validation loss as the optimization objective.

The results are shown in~\cref{fig:logistic_results}. Even though the input space is fairly high-dimensional with~$d=4$, all algorithms determine close-to-optimal inputs quickly. In particular, MAP estimation determines that both the dropout rate and the batch size do not influence the validation loss significantly. Since the theoretical results in~\textsc{A-GP-UCB} are compatible with MAP estimation, our approach achieves the same empirical performance, but has theoretical worst-case regret bounds. After convergence, the BO baselines repeatedly evaluate the same inputs, without gaining any new information. In contrast, our method continues to explore in order to potentially find better inputs. While it does not occur in this case, this allows us to be more confident that the global optimum has been identified as~$t$ increases. For standard BO methods, there is no guarantee of convergence with misspecified hyperparameters.

%!TEX root = ../root.tex

\section{Conclusion and Future Work}
\label{sec:conclusion}

We introduced~\algabbrv, a BO algorithm that is provably no-regret when hyperparameters are unknown. Our method adapts the hyperparameters online, which causes the underlying BO algorithm to consider larger function spaces over time. Eventually, the function space is large enough to contain the true function, so that our algorithm provably converges. We evaluated our method on several benchmark problems, confirming that, on the one hand, it provably converges even in cases where standard BO algorithms get stuck in local optima, and, on the other hand, enjoys competitive performance as standard BO algorithms that do not have theoretical guarantees in this setting.

The main idea behind our analysis is that adapting the hyperparameters increases the cumulative regret bound, but we do so slowly enough to converge eventually. This idea is fairly general and could also be applied to other no-regret algorithms. Another potential future direction is to investigate alternative strategies to select the scaling factors~$b(t)$ and~$g(t)$ and consider adapting other parameters such as the kernel structure.

% % Acknowledgements should go at the end, before appendices and references

\acks{This research was supported in part by SNSF grant {200020\_159557}, ERC grant no. {815943}, NSERC grant {RGPIN-2014-04634}, the Vector Institute, and an Open Philantropy Project AI fellowship. We would like to thank Johannes Kirschner for valueable discussions.}


% \newpage

% \vskip 0.2in
% \bibliography{sample}}
\bibliography{root.bib}

\appendix
%!TEX root = ../root.tex

\section{Proof of Main Theorem}

\begin{lemma}
  Let~$f \in \mathcal{H}_{\theta_{t^*}}$ with~$\|f\|_{\theta_{t^*}} \leq B_{t^*}$. Then, for any monotonically increasing functions~$g(t) \geq 1$ and~$b(t) \geq 1$ and for all~$t \geq t^*$: $f \in \mathcal{H}_{\theta_{t}}$ with~$\|f\|_{\theta_{t}} \leq B_{t}$
  \label{lem:f_contained_once_found}
\end{lemma}
\begin{proof}
  \cref{thm:rkhs_norm_change} together with monotonicity of~$g$ yields $\mathcal{H}_{\theta_{t}} \supseteq \mathcal{H}_{\theta_{t^*}}$ so that~$f \in \mathcal{H}_{\theta_{t}}$ and
$$
  \|f\|_{\theta_t}
  \leq \prod_{1 \leq i \leq d} \frac{[\theta_{t^*}]_i}{ [\theta_t]_i} \|f\|_{\theta_{t^*}}
  \leq \frac{g(t)^d} {g(t^*)^{d}} B_{t^*}
  = \frac{g(t)^d} {g(t^*)^{d}} g(t^*)^d b(t^*) B_0
  = g(t)^d b(t^*) B_0
  \leq B_t
$$
\end{proof}

\begin{lemma}
Under the assumptions of~\cref{thm:confidence_interval}, let~$\theta_t$ be a predictable sequence of kernel hyperparameters such that~$\|f\|_{k_{\theta_t}} \leq B_t$ and let the GP predictions~$\mu_t$ and~$\sigma_t$ use the prior covariance~$k_{\theta_t}$. If $\beta_t^{1/2} = B_t + 4 \sigma \sqrt{ I_{\theta_t}(\mb{y}_{t}; f) + 1 + \mathrm{ln}(1 / \delta)}$, then
$|\, f(\mb{x}) - \mu_{t}(\mb{x}) \,| \leq \beta_{t}^{1/2} \sigma_{t}(\mb{x})
$
holds for all~${\mb{x} \in D}$ and iterations~${t \geq 0}$ jointly with probability at least~${1 - \delta}$.
\label{thm:confidence_interval_extended}
\end{lemma}
\begin{proof}
 The proof is the same as the one by \citet{Abbasi-Yadkori2012Online,Chowdhury2017Kernelized}, except that the kernel is time-dependent.
\end{proof}
We are now ready to prove the main result:
\newline

\begin{proof}[\cref{thm:main}]
  We split the regret bound into two terms,~$R_t = t_0\, r_{c} + r_s(t)$.
  In the initial rounds, where either~$B_t \leq g(t)^d B_0$ or~$\max_i [\theta]_i / [\theta]_0 > 1$, the regret is trivially bounded by~$r_t \leq 2\|f\|_\infty \leq 2 \|f\|_{\theta} \leq B$. Thus~$r_c \leq 2 B$. Let~$t_0 \in (0, \infty]$ be the first iteration such that~$f \in \mathcal{H}_{\theta_{t_0}}$ with~$\|f\|_{\theta_{t_0}} \leq B_{t_0}$. From~\cref{lem:f_contained_once_found}, we have that $f \in \mathcal{H}_{\theta_{t}}$ with~$\|f\|_{\theta_{t}} \leq B_t$ for all~$t\geq t_0$. Thus we can use~\cref{thm:confidence_interval_extended} to conclude~$|f - \mu_t(\mb{x})| \leq \beta_t^{1/2} \sigma_t(\mb{x})$ for all~$\mb{x} \in \mathcal{D}$ and~$t\geq t_0$ jointly with probability at least~$(1-\delta)$. We use Lemmas 5.2-5.4 in \citet{Srinivas2012Gaussian} to conclude that the second stage has a regret bound of~$r_s^2(t) \leq C_1 \beta_t I(\mb{y}_t; f) $, which concludes the proof.
\end{proof}

\section{Bound on the information capacity~$\gamma_t$}

\begin{theorem}[Theorem 8 in \cite{Srinivas2012Gaussian}]
Suppose that ${D \subset \mathbb{R}^d}$ is compact, and~${k(\mb{x}, \mb{x}')}$ is a covariance function for which the additional assumption of Theorem 2 in~\cite{Srinivas2012Gaussian} hold. Moreover, let~${ B_k(T_*) = \sum_{s > T_*} \lambda_s }$, where~${\{\lambda_s\}}$ is the operator spectrum of~$k$ with respect to the uniform distribution over~$D$. Pick~$\tau > 0$, and let~${n_T = C_4 T^\tau (\log T) }$ with~${C_4 = 2 \mathcal{V}(D) (2 \tau + 1)}$. Then, the following bound holds true:
%
% \begin{equation}
%   \begin{aligned}
% \gamma_T \leq& \frac{1/2}{1 - e^{-1}} \, \max_{r \in \{1,\dots,T\}}
% \big(
% T_* \log(r n_T / \sigma^2) \\
% &+ C_4 \sigma^{-2} (1 - r / T) (\log T) ( T^{\tau + 1} B_k(T_*) + 1) \\
% &+ \bigO(T^{1 - \tau / d})
% \big).
% \end{aligned}
% \label{eq:gamma_t_spectrum_bound}
% \end{equation}
\begin{equation}
\gamma_T \leq \frac{1/2}{1 - e^{-1}} \, \max_{r \in \{1,\dots,T\}}
T_* \log\left(\frac{r n_T}{\sigma^2}\right)
+ C_4 \sigma^{-2} (1 - \frac{r}{T}) ( B_k(T_*) T^{\tau + 1} + 1) \log T
+ \bigO(T^{1 - \frac{\tau}{d}}).
\label{eq:gamma_t_spectrum_bound}
\end{equation}
\label{thm:bound_gamma_with_spectrum}
\end{theorem}

\cref{thm:bound_gamma_with_spectrum} allows us to bound~$\gamma_t$ through the operator spectrum of the kernel with respect to the uniform distribution. We now consider this quantity for two specific kernels.

\subsection{Bounds for the Squared Exponential Kernel}

\begin{lemma}
\label{thm:log_1_x_bound}
For all~$x \in [0, x_\mathrm{max}^2]$ it holds that
% \begin{equation}
$\log(1 + x^2) \geq \frac{\log(1 + x_\mathrm{max}^2)}{x_\mathrm{max}^2} x^2$
% \end{equation}
\end{lemma}

In this section, we use~\cref{thm:bound_gamma_with_spectrum} to obtain concrete bounds for the Gaussian kernel. From~\cite{Seeger2008Information}, we obtain a bound on the eigenspectrum that is given by
%
%
% with
%
\begin{equation*}
\lambda_s \leq cB^{s^{1/d}}, \textnormal{~where~} c = \sqrt{\frac{2 a}{A}}, \quad b = \frac{1}{2 \theta_t^2}, \quad B = \frac{b}{A},\quad \textnormal{and}\quad A = a + b + \sqrt{a^2 + 2ab }.
\end{equation*}
The constant~$a>0$ parameterizes the distribution~${\mu(\mb{x}) \sim \mathcal{N}(\mb{0}, (4a)^{-1} \mb{I}_d)}$.
As a consequence of~$\theta_t > 0$, we have that~$b \geq 0$, $0<B<1$, $c > 0$, and $A > 0$.
In the following, we bound the eigenspectrum. The steps follow the outline of~\cite{Seeger2008Information}, but we provide more details and the dependence on the lengtscales~$\theta_t$ is made explicit:
\begin{align*}
  B_k(T_*) &= \sum_{s > T_*} \lambda_s
  \leq c \sum_{s \geq T_* + 1}B^{s^{1/d}}
  = c \sum_{s \geq T_* + 1} \exp \log (B^{s^{1/d}})
  = c \sum_{s \geq T_* + 1} \exp(s^{1/d} \log B ) ,\\
  &= c \sum_{s \geq T_* + 1} \exp(- s^{1/d}  \alpha )
  \leq c \int_{T_*}^\infty \exp(- \alpha s^{1/d} ) \dif s,
  \intertext{
  where $\alpha = -\log B$. Now substitute $s = \phi(t) = (t/\alpha)^d$. Then $\dif s = \frac{d t^{d-1}}{\alpha} \dif t$ and}
  B_k(T_*) &\leq c \int_{\alpha T_*^{1/d}}^\infty \exp(-t)  \frac{d t^{d-1}}{\alpha} \dif t
  = c d \alpha^{-d} \Gamma(d, \alpha T_*^{1/d}),
  \intertext{
  where $\Gamma(d, \beta) = \int_\beta^\infty e^{-t} t^{d-1} \,dt = (d-1)! e^{-\beta} \sum_{k=0}^{d-1} \beta^k / k!$ for $d \in \mathbb{N}$ as in~\citet[(8.352.4)]{Gradshtein2007Table}. Then, with $\beta = \alpha T_*^{1/d}$,
  }
  B_k(T_*) &\leq c d \alpha^{-d} (d - 1)! e^{-\beta} \sum_{k=0}^{d-1} \beta^k / k!
  = c (d!) \alpha^{-d} e^{-\beta} \sum_{k=0}^{d-1} (k!)^{-1} \beta^k.
\end{align*}
%

% \begin{equation}
% B_k(T_*) \leq c(d!)\alpha^{-d} e^{-\beta} \sum_{j = 0}^{d-1} (j!)^{-1} \beta^{j},
% \end{equation}
%
% where~${\alpha = -\log B > 0}$, ${\beta = \alpha T_*^{1/d} }$.

Before we bound the information gain, let us determine how~$\alpha^{-d}$ and~$c$ depend on the lengthscales. In particular, we want to quantify their upper bounds in terms of~$g(t)$.
\begin{align}
\alpha^{-d} &= \log^{-d} (1 / B )
= \log^{-d} \left( 2 \theta_t^2 A \right)
% &= \log^{-d} \left( 2 \theta_t^2 \left(a + \frac{1}{2\theta_t^2} + \sqrt{a^2 + \frac{a}{\theta_t^2}} \right) \right) \\
= \log^{-d} \left( 1 + 2 \theta_t^2 a + 2 \theta_t  \sqrt{a^2 + \frac{a}{\theta_t^2}} \right) \\
&\leq \log^{-d} \left( 1 + 2 \theta_t^2 a \right)
\leq \left( \frac{\log(1 + 2 \theta_0^2 a)}{2\theta_0^2 a} 2 \theta_t^2 a \right)^{-d} \label{eq:bound_alpha_d_Lemma} \text{~~~~~~~~by \cref{thm:log_1_x_bound}} \\
&= \bigO\left( \theta_t^{-2d} \right)
= \bigO\left( g^{2d}(t) \right),
\end{align}
where~\cref{eq:bound_alpha_d_Lemma} follows from~\cref{thm:log_1_x_bound}, since $g(t) \geq 1$ for all~$t>0$.
Similarly,
\begin{equation}
c = \left( \frac{2a}{a + \frac{1}{2\theta_t^2} + \sqrt{a^2 + \frac{a}{\theta_t^2}}} \right)^{d/2}
\leq \left( \frac{2a}{\frac{1}{2 \theta_t^2}} \right)
= \left( 4a\theta_t^2 \right)^{d/2}
= \bigO(g(t)^{-d}).
\end{equation}

As in~\cite{Srinivas2012Gaussian}, we choose~$T_* = (\log(T n_T) / \alpha)^d$, so that~$\beta=\log(T n_T)$ and therefore does not depend on~$g_t$.
Plugging into~\cref{eq:gamma_t_spectrum_bound}, the first term of~\cref{eq:gamma_t_spectrum_bound} dominates and
\begin{align}
\gamma_T = \bigO \left( \left[ \log(T^{d+1} (\log T)) \right]^{d+1} c \alpha^{-d}  \right)^{d/2}
= \bigO \left( (\log T)^{d + 1} g(t)^d  \right).
\end{align}

\subsection{Mat\'ern kernel}

Following the proof for Theorem 2 in the addendum to \citet{Seeger2008Information}, we have that
\begin{equation}
  \lambda_s^{(T)} \leq C (1 + \delta) s^{-(2 \nu + d) / d} ~ \forall s \geq s_0,
\end{equation}

For the leading constant we have~$C=C_3^{(2 \nu + d) / d}$ with~$\alpha = \frac{2 \pi \theta_t}{\sqrt{2 \nu}}$. Hiding terms that do not depend on~$\alpha$ and therefore~$g(t)$, we have
\begin{align*}
  &C_t(\alpha, \nu) = \frac{\Gamma(\nu + d / 2)}{\pi^{d/2} \Gamma(\nu)} \alpha^d = \bigO(g(t)^{-d})
  &&c_1 = \frac{1}{(2\pi)^d C_t(\alpha, \nu)} = \bigO(g(t)^d) \\
  &C_2 = \frac{\alpha^{-d}}{2^d \pi^{d/2}\Gamma(d/2)} = \bigO(g(t)^{d})
  &&C_3 = C_2 \frac{2\tilde{C}}{d}  c_1^\frac{-d}{2 \nu + d} = \bigO( g(t)^d g(t)^\frac{-d^2}{2 \nu + d} ) = \bigO(g(t)^d),
\end{align*}
so that $C = \bigO(g(t)^{2 \nu + d})$. The second term in~$C_3$ must be over-approximated as a consequence of the proof strategy. It follows that
$
  B_k(T_*) = \bigO(g(t)^{2 \nu d} T_*^{1-(2 \nu + d) / d})
$
and, as in~\cite{Srinivas2012Gaussian}, that
$
  \gamma_T = \bigO(T^\frac{d(d+1)}{2\nu + d(d+1)} (\log T) g(t)^{2 \nu d}).
$

\end{document}